\documentclass{article} 
\usepackage{iclr2020_conference,times}
\iclrfinalcopy

\usepackage{hyperref}
\usepackage{url}
\usepackage{graphicx}
\usepackage{caption}
\usepackage{subcaption}
\usepackage{algorithm}
\usepackage[noend]{algpseudocode}
\usepackage{amsmath, amssymb, amsthm}
\usepackage{physics}
\usepackage{comment}

\newcommand{\myalgname}{LiPopt}
\newcommand{\bfmyalgname}{\textbf{LiPopt}}

\newcommand{\R}{\mathbb{R}} 
\newcommand{\N}{\mathbb{N}} 
\newcommand{\bigo}{\mathcal{O}} 

\DeclareMathOperator*{\Diag}{Diag} 
\DeclareMathOperator*{\vect}{vec} 
\DeclareMathOperator*{\supp}{supp} 

\newcommand{\Dom}{\mathcal{X}} 

\newcommand{\normi}[1]{\|{#1}\|}

\newtheorem{theorem}{Theorem} 
\newtheorem*{theorem*}{Theorem}
\newtheorem*{proposition*}{Proposition}

\newtheorem{definition}{Definition}

\newtheorem{proposition}{Proposition}

\title{Lipschitz constant estimation of Neural Networks via sparse polynomial
optimization}

\author{%
    Fabian Latorre, Paul Rolland and Volkan Cevher \\
    EPFL, Switzerland \\
    \texttt{firstname.lastname@epfl.ch}
}

\begin{document}

\maketitle

\begin{abstract}
We introduce \myalgname, a polynomial optimization framework for computing
    increasingly tighter upper bounds on the Lipschitz constant of neural
    networks. The underlying optimization problems boil down to either linear
    (LP) or semidefinite (SDP) programming. We show how to use the sparse
    connectivity of a network, to significantly reduce the complexity of
    computation. This is specially useful for convolutional as well as pruned
    neural networks. We conduct experiments on networks with random weights as
    well as networks trained on MNIST, showing that in the particular case of
    the $\ell_\infty$-Lipschitz constant, our approach yields superior
    estimates, compared to baselines available in the literature.
\end{abstract}

\section{Introduction}
\label{sec:introduction}
We consider a neural network $f_d$ defined by the recursion:
\begin{equation}
    \label{eq:fn_def}
    f_1(x):=W_1 x \qquad  f_i(x):= W_i \hspace{0.3mm} \sigma(f_{i-1}(x)), \quad i=2,\ldots, d
\end{equation}
for an integer $d$ larger than 1, matrices $\{W_i\}_{i=1}^d$ of appropriate
dimensions and an \textit{activation function} $\sigma$, understood to be
applied element-wise. We refer to $d$ as the \emph{depth}, and we focus on
the case where $f_d$ has a single real value as output.

In this work, we address the problem of estimating the \textit{Lipschitz
constant} of the network $f_d$. A function $f$ is
\textit{Lipschitz continuous} with respect to a norm $\lVert \cdot \rVert$ if
there exists a constant $L$ such that for all $x,y$ we have $|f(x)- f(y)|
\leq L \lVert x - y \rVert$. The minimum over all such values satisfying this
condition is called the \textit{Lipschitz constant} of $f$ and is denoted
by $L(f)$.

The Lipschitz constant of a neural network is of major importance in many
successful applications of \emph{deep learning}. In the context of supervised
learning, \citet{Bartlett2017} show how it directly correlates with the
generalization ability of neural network classifiers, suggesting it as model
complexity measure. It also provides a measure of robustness against
adversarial perturbations \citep{Szegedy2013} and can be used to improve such
metric \citep{Cisse2017}. Moreover, an upper bound on $L(f_d)$ provides a
certificate of robust classification around data points \citep{Weng2018}.

Another example is the discriminator network of the
\emph{Wasserstein GAN} \citep{Arjovsky2017}, whose Lipschitz constant is
constrained to be at most 1. To handle this constraint, researchers have
proposed different methods like heuristic penalties \citep{Gulrajani2017},
upper bounds \citep{Miyato2018}, choice of activation function
\citep{Anil2019}, among many others. This line of work has shown that accurate
estimation of such constant is key to generating high quality images.

Lower bounds or heuristic estimates of $L(f_d)$ can be used to provide a
general sense of how robust a network is, but fail to provide true certificates
of robustness to input perturbations. Such certificates require true
upper bounds, and are paramount when deploying safety-critical deep
reinforcement learning applications \citep{Berkenkamp2017,Jin2018}.
The trivial upper bound given by the product of layer-wise
Lipschitz constants is easy to compute but rather loose and overly
pessimistic, providing poor insight into the true robustness of a network
\citep{Huster2018}.

Indeed, there is a growing need for methods that provide tighter upper bounds
on $L(f_d)$, even at the expense of increased complexity. For example
\citet{Raghunathan2018a,Jin2018,Fazlyab2019} derive upper bounds based on
\textit{semidefinite programming (SDP)}. While expensive to compute, these type
of certificates are in practice surprisingly tight. Our work belongs
in this vein of research, and aims to overcome some limitations in the current
state-of-the-art.

\textbf{Our Contributions.}
\begin{itemize}
    \item We present \bfmyalgname, a general approach for upper
        bounding the Lipschitz constant of a neural network based on a
        relaxation to a \emph{polynomial optimization problem (POP)}
        \citep{Lasserre2015}. This approach requires only that the unit ball be
        described with polynomial inequalities, which covers the common
        $\ell_2$- and $\ell_\infty$-norms.
    \item Based on a theorem due to \citet{Weisser2018}, we exploit the sparse
        connectivity of neural network architectures to derive a sequence of
        linear programs (LPs) of considerably smaller size than their vanilla
        counterparts. We provide an asymptotic analysis of the size of
        such programs, in terms of the number of neurons, depth and sparsity of
        the network.
    \item Focusing on the $\ell_\infty$-norm, we experiment
        on networks with random weights and networks trained on MNIST
        \citep{Lecun1998}. We evaluate different configurations of depth, width
        and sparsity and we show that the proposed sequence of LPs can
        provide tighter upper bounds on $L(f_d)$ compared to other 
        baselines available in the literature.
\end{itemize}

\textbf{Notation. }We denote by $n_i$ the number of columns of the matrix $W_i$
in the definition \eqref{eq:fn_def} of the network. This corresponds to the
size of the $i$-th layer, where we identify the input as the first layer. We
let $n=n_1+\ldots+n_{d}$ be the total number of neurons in the network. For
a vector $x$, $\Diag(x)$ denotes the square matrix with $x$ in its diagonal and
zeros everywhere else. For an array $X$, $\vect(X)$ is the
\emph{flattened} array. The support of a sequence $\supp(\alpha)$ is defined
as the set of indices $j$ such that $\alpha_j$ is nonzero. For
$x=[x_1,\ldots,x_n]$ and a sequence of nonnegative integers
$\gamma=[\gamma_1,\ldots,\gamma_n]$ we denote by $x^\gamma$ the monomial
$x_1^{\gamma_1}x_2^{\gamma_2}\ldots x_n^{\gamma_n}$. The set of nonnegative
integers is denoted by $\N$.

\textbf{Remark. }The definition of network \eqref{eq:fn_def} covers typical
architectures composed of dense and convolutional layers. In general, our
proposed approach can be readily extended with minor modifications to any
directed acyclic computation graph e.g., residual network architectures
\citep{He2016}.
%
%

\section{Polynomial optimization formulation}
\label{sec:formulation}
In this section we derive an upper bound on $L(f_d)$ given by the value of a
POP, i.e. the minimum value of a polynomial subject to polynomial inequalities.
Our starting point is the following theorem, which casts $L(f)$ as an
optimization problem:
\begin{theorem}
    \label{thm:main}
    Let $f$ be a differentiable and Lipschitz continuous function on an open,
    convex subset $\Dom$ of an euclidean space. Let $\normi{\cdot}_*$ be the dual norm. The
    Lipschitz constant of $f$ is given by
    \begin{equation}
        \label{eq:lips_sup_derivative}
        L(f)=\sup_{x \in \Dom} \norm{\nabla{f}(x)}_*
    \end{equation}
\end{theorem}
For completeness, we provide a proof in appendix \ref{app:lips}. In our setting,
we assume that the activation function $\sigma$ is Lipschitz continuous and
differentiable. In this case, the assumptions of \autoref{thm:main} are
fulfilled because $f_d$ is a composition of activations and linear
transformations. The differentiability assumption rules out the common ReLU
activation $\sigma(x)=\max\{0, x\}$, but allows many others such as the
exponential linear unit (ELU) \citep{Clevert2015} or the softplus.

Using the chain rule, the compositional structure of $f_d$ yields the following
formula for its gradient:
\begin{equation}
    \label{eq:grad_formula}
    \nabla f_d(x) = W_1^T \prod_{i=1}^{d-1} \Diag(\sigma'(f_{i}(x))) W_{i+1}^T 
\end{equation}
For every $i=1,\ldots,d-1$ we introduce a variable $s_i=\sigma'(f_i(x))$
corresponding to the derivative of $\sigma$ at the $i$-th hidden layer of the
network. For activation functions like ELU or softplus, their derivative
is bounded between $0$ and $1$, which implies that $0 \leq s_i \leq
1$. This bound together with the definition of the dual norm $\normi{x}_* :=
\sup_{\normi{t} \leq 1} t^Tx$ implies the following upper bound of $L(f_d)$:
\begin{equation}
    \label{eq:upper_bound}
    L(f_d) \leq
    \max \left \{  t^T W_1^T \prod_{i=1}^{d-1} \Diag(s_i)  W_{i+1}^T:
    0 \leq s_i \leq 1, \normi{t} \leq 1 \right \}
\end{equation}
We will refer to the polynomial objective of this problem as the
\emph{\textbf{norm-gradient polynomial}} of the network $f_d$, a central
object of study in this work.

For some frequently used $\ell_p$-norms, the constraint $\|t\|_p \leq 1$ can be
written with polynomial inequalities. In the rest of this work, \textbf{we use
exclusively the $\ell_\infty$-norm} for which $\normi{t}_\infty \leq 1$ is
equivalent to the polynomial inequalities $-1 \leq t_i \leq 1$, for
$i=1,\ldots,n_1$. However, note that when $p \geq 2$ is a positive even integer,
$\normi{t}_p \leq 1$ is equivalent to a single polynomial inequality
$\normi{t}_p^p \leq 1$, and our proposed approach can be adapted with minimal
modifications.

In such cases, the optimization problem in the right-hand side of
\eqref{eq:upper_bound} is a POP. Optimization of polynomials is a NP-hard
problem and we do not expect to have efficient algorithms for solving
\eqref{eq:upper_bound} in this general form. In the next sections we describe
\bfmyalgname: a systematic way of obtaining an upper bound on $L(f_d)$ via
tractable approximation methods of the POP \eqref{eq:upper_bound}.

\textbf{Local estimation. }In many practical escenarios, we have additional
bounds on the input of the network. For example, in the case of image
classification tasks, valid input is constrained in a hypercube. In the
robustness certification task, we are interested in all possible input in a
$\epsilon$-ball around some data point. In those cases, it is interesting to
compute a \textit{local Lipschitz constant}, that is, the Lipschitz constant of
a function restricted to a subset.

We can achieve this by deriving tighter bounds $0 \leq l_i \leq s_i \leq u_i
\leq 1$, as a consequence of the restricted input (see for example, Algorithm 1
in \citet{Wong2018}). By incorporating this knowledge in the optimization
problem \eqref{eq:upper_bound} we obtain bounds on local Lipschitz constants of
$f_d$. We study this setting and provide numerical experiments in section
\ref{sec:exp-local}.

\textbf{Choice of norm. }We highlight the importance of computing good upper
bounds on $L(f_d)$ with respect to the $\ell_\infty$-norm. It is one of the
most commonly used norms to assess robustness in the adversarial examples
literature. Moreover, it has been shown that, in practice, $\ell_\infty$-norm robust
networks are also robust in other more plausible measures of perceptibility,
like the Wasserstein distance \citep{Wong2019}. This motivates our focus
on this choice.

\section{Hierarchical solution based on a Polynomial Positivity certificate}
\label{sec:lp_certificate}
For ease of exposition, we rewrite \eqref{eq:upper_bound} as a POP
constrained in $[0,1]^{n}$ using the substitution $s_0:=(t+1)/2$.
Denote by $p$ the norm-gradient polynomial, and let $x=[s_0, \ldots, s_{d-1}]$
be the concatenation of all variables. Polynomial optimization methods
\citep{Lasserre2015} start from the observation that a value $\lambda$
is an upper bound for $p$ over a set $K$ if and only if the polynomial $\lambda
- p$ is positive over $K$.

In \bfmyalgname, we will employ a well-known classical result in
algebraic geometry, the so-called \emph{Krivine's positivity
certificate}\footnote{also known as \emph{Krivine's Positivstellensatz}},
but in theory we can use
any positivity certificate like sum-of-squares (SOS). The following
is a straightforward adaptation of Krivine's certificate to our 
setting:
\begin{theorem}{(Adapted from \citet{Krivine1964,Stengle1974,Handelman1988})}
\label{thm:krivine}
If the polynomial $\lambda-p$ is strictly positive on $[0,1]^n$, then there
exist finitely many positive weights $c_{\alpha \beta}$ such that
\begin{equation}
    \label{eq:krivine}
    \lambda - p = \sum_{(\alpha, \beta) \in \N^{2n}} c_{\alpha \beta} h_{\alpha \beta}, \qquad
    h_{\alpha \beta}(x) := \prod_{j=1}^n x_j^{\alpha_j}(1-x_j)^{\beta_j}
\end{equation}
\end{theorem}
By truncating the degree of Krivine's positivity certificate
(\autoref{thm:krivine}) and minimizing over all possible upper bounds $\lambda$
we obtain a \textit{hierarchy} of LP problems \citep[Section 9]{Lasserre2015}:
\begin{equation}
    \label{eq:lp_prototype}
    \theta_k := \min_{c \geq 0, \lambda} \left \{
    \lambda: \lambda - p=\sum_{(\alpha, \beta) \in \N_{k}^{2n}}
    c_{\alpha\beta} h_{\alpha \beta} \right \}
\end{equation}
where $\N_{k}^{2n}$ is the set of nonnegative integer sequences of length $2n$
adding up to at most $k$. This is indeed a sequence of LPs as the polynomial
equality constraint can be implemented by equating coefficients in the
canonical monomial basis. For this polynomial equality to be feasible, the
degree of the certificate has to be at least that of the norm-gradient
polynomial $p$, which is equal to the depth $d$. This implies that the first
nontrivial bound ($\theta_k < \infty$) corresponds to $k=d$.

The sequence $\{\theta_k\}_{k=1}^\infty$ is non-incresing and converges to the
maximum of the upper bound \eqref{eq:upper_bound}. Note that for any level of
the hierarchy, the solution of the LP \eqref{eq:lp_prototype} provides a
valid upper bound on $L(f_d)$. 

An advantage of using Krivine's positivity certificate over SOS is that one
obtains an LP hierarchy (rather than SDP), for which commercial solvers can
reliably handle a large instances. Other positivity certificates offering a
similar advantage are the DSOS and SDSOS hierarchies \citep{Ahmadi2019}, which
boil down to LP or \textit{second order cone programming} (SOCP), respectively.

\textbf{Drawback. }The size of the LPs given by Krivine's positivity
certificate can become quite large. The dimension of the variable $c$ is
$|\N_{k}^{2n}|=\bigo(n^k)$. For reference, if we consider the MNIST dataset and
a one-hidden-layer network with 100 neurons we have $\abs{\N_{2}^{2n}} \approx
1.5 \times 10^6$ while $\abs{\N_{3}^{2n}} \approx 9.3 \times 10^{8}$. To make
this approach more scalable, in the next section we exploit the
sparsity of the polynomial $p$ to find LPs of drastically smaller size than
\eqref{eq:lp_prototype}, but with similar approximation properties. 

\textbf{Remark. }In order to compute upper bounds for local Lipschitz
constants, first obtain tighter bounds $0 \leq l_i \leq s_i \leq u_i$ and then
perform the change of variables $\widetilde{s}_i = (s_i - l_i)/(u_i - l_i)$ to
rewrite the problem \eqref{eq:upper_bound} as a POP constrained on $[0,1]^n$.

\section{Reducing the number of variables}
\label{sec:sparse}
Many neural network architectures, like those composed of convolutional layers,
have a highly sparse connectivity between neurons. Moreover, it has been
empirically observed that up to 90\% of network weights can be \emph{pruned}
(set to zero) without harming accuracy \citep{Frankle2019}. In such cases their
norm-gradient polynomial has a special structure that allows polynomial
positivity certificates of smaller size than the one given by Krivine's positivity
certificate (\autoref{thm:krivine}).

In this section, we describe an implementation of \bfmyalgname
\,(Algorithm \ref{alg:poplip-krivine}) that exploits the sparsity of the network
to decrease the complexity of the LPs \eqref{eq:lp_prototype} given by the
Krivine's positivity certificate. In this way, we obtain upper bounds on
$L(f_d)$ that require less computation and memory. Let us start with
the definition of a \emph{valid sparsity pattern}:
\begin{definition}
\label{def:sparsity_pattern}
Let $I=\{1, \ldots, n\}$ and $p$ be a polynomial with variable $x \in \R^n$. A
    valid sparsity pattern of $p$ is a sequence $\{I_i\}_{i=1}^m$ of subsets of
    $I$, called cliques, such that $\bigcup_{i=1}^m I_i = I$ and:
\begin{itemize}
\item $p = \sum_{i=1}^m p_i$ where $p_i$ is a polynomial that
    depends only on the variables $\{x_j : j \in I_i\}$
\item for all $i=1,\ldots,m-1$ there is an $l \leq i$ such that
    $(I_{i+1} \cap \bigcup_{r=1}^i I_r ) \subseteq I_l$
\end{itemize}
\end{definition}
When the polynomial objective $p$ in a POP has a valid sparsity pattern, there
is an extension of \autoref{thm:krivine} due to \citet{Weisser2018}, providing
a smaller positivity certificate for $\lambda - p$ over $[0,1]^n$. We refer to
it as the \emph{sparse Krivine's certificate} and we include it here for completeness:
\begin{theorem}[Adapted from \citet{Weisser2018}]
\label{thm:sparse-krivine}
Let a polynomial $p$ have a valid sparsity pattern $\{I_i\}_{i=1}^m$. Define $N_i$
as the set of sequences $(\alpha, \beta) \in \N^{2n}$ where the support of
both $\alpha$ and $\beta$ is contained in $I_i$. If $\lambda - p$ is
strictly positive over $K=[0, 1]^n$, there exist finitely many positive weights
$c_{\alpha \beta}$ such that
\begin{align}
    \label{eq:sparse_cert}
    \lambda - p=\sum_{i=1}^m h_i, \qquad & \qquad
    h_i = \sum_{(\alpha, \beta) \in N_i} c_{\alpha\beta} h_{\alpha \beta}
\end{align}
where the polynomials $h_{\alpha\beta}$ are defined as in \eqref{eq:krivine}.
\end{theorem}
The sparse Krivine's certificate can be used like the general version
(\autoref{thm:krivine}) to derive a sequence of LPs approximating the upper
bound on $L(f_d)$ stated in \eqref{eq:upper_bound}. However, the number of
different polynomials $h_{\alpha\beta}$ of degree at most $k$ appearing in
the sparse certificate can be drastically smaller, the amount of which
determines how \emph{good} the sparsity pattern is.

We introduce a graph that depends on the network $f_d$, from which we will
extract a sparsity pattern for the norm-gradient polynomial of a network.
\begin{definition}
    Let $f_d$ be a network with weights $\{W_i\}_{i=1}^d$. Define a directed
    graph $G_d=(V,E)$ as:
    \begin{equation}
    \begin{gathered}
        V=\left \{s_{i,j}: 0
    \leq i \leq d-1, \, 1 \leq j \leq n_i \right \} \\
        E=\left \{(s_{i,j}, s_{i+1,k}): 0 \leq i \leq d-2, [W_i]_{k,j} \neq 0  \right\} 
    \end{gathered}
    \end{equation}
which we call the computational graph of the network $f_d$.
\end{definition}
In the graph $G_d$ the vertex $s_{(i,j)}$ represents the $j$-th neuron in the
$i$-th layer. There is a directed edge between two neurons in consecutive
layers if they are joined by a nonzero weight in the network. The following
result shows that for fully connected networks we can extract a valid sparsity
pattern from this graph. We relegate the proof to appendix
\ref{app:proof_sparsity}.
\begin{proposition}
\label{prop:sparsity_pattern}
    Let $f_d$ be a dense network (all weights are nonzero). The following sets,
    indexed by $i=1,\ldots,n_d$, form a valid sparsity pattern for the norm-gradient
    polynomial of the network $f_d$:
    \begin{equation}
        \label{eq:sparsity_pattern}
        I_{i}:= \left \{s_{(d-1, i)}\} \cup \{s_{(j,k)}: \text{ there exists a
        directed path from }s_{(j,k)} \text{ to }  s_{(d-1,i)} \text{ in } G_d \right \}
    \end{equation}
\end{proposition}
We refer to this as the \emph{sparsity pattern induced by
$G_d$}. An example is depicted in in Figure \ref{fig:sparsity_pattern}. 
\begin{figure}
\centering
\begin{minipage}{.47\textwidth}
  \centering
    \includegraphics[width=0.6\textwidth]{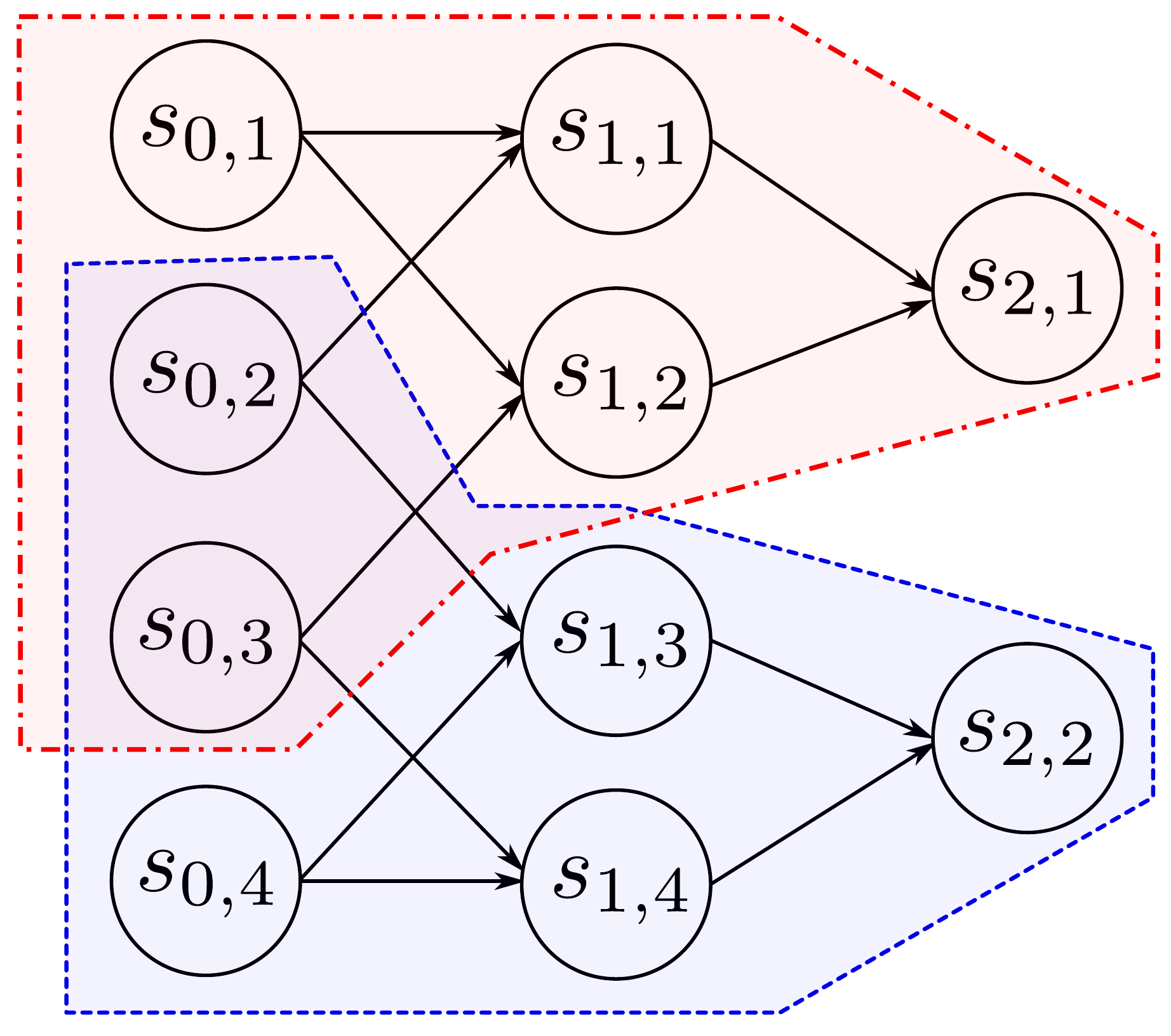}
    \captionof{figure}{Sparsity pattern of Proposition
    \ref{prop:sparsity_pattern} for a network of depth three.}
  \label{fig:sparsity_pattern}
\end{minipage}%
\qquad
\begin{minipage}{.47\textwidth}
  \centering
    \includegraphics[width=0.6\textwidth]{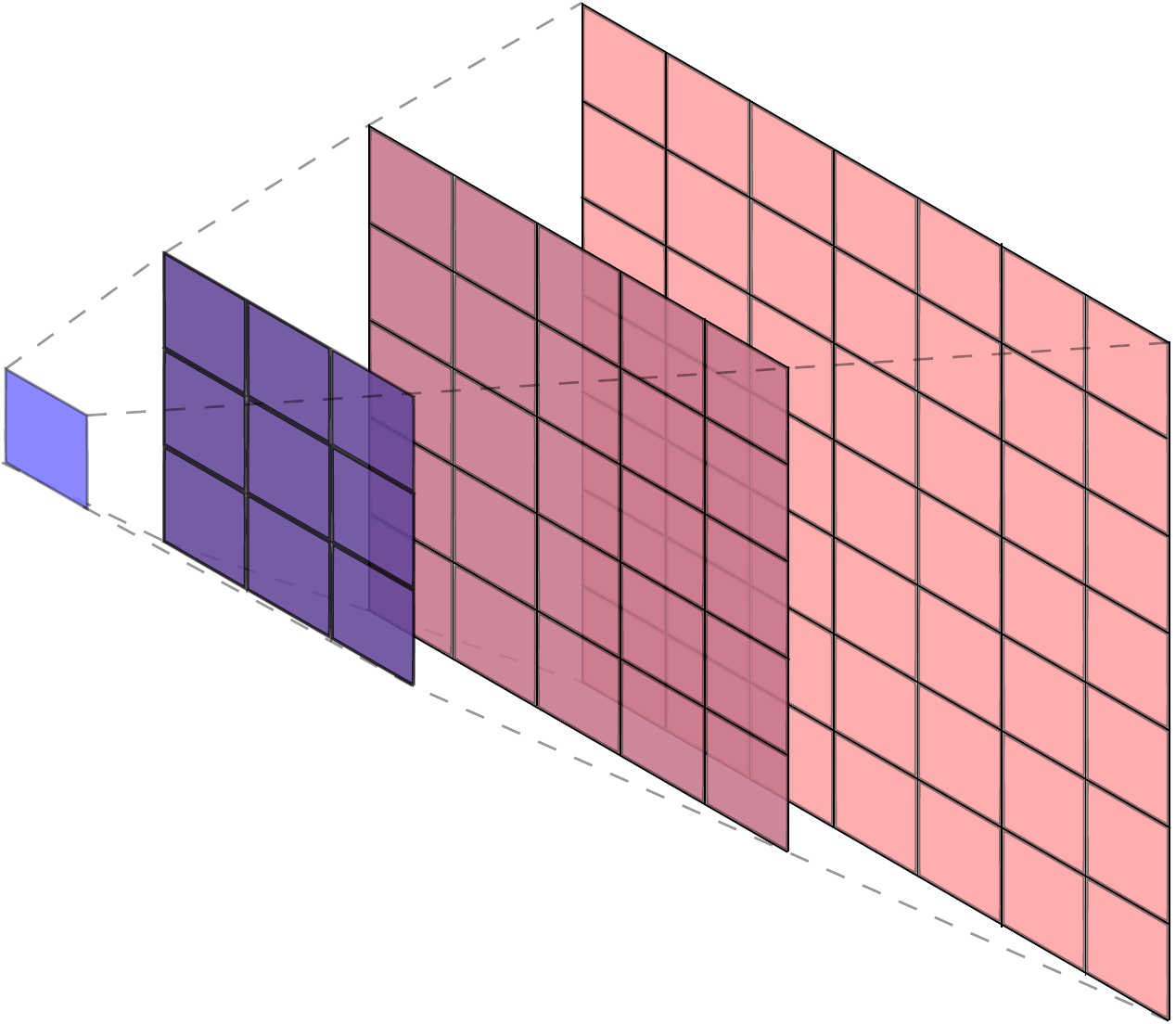}
    \captionof{figure}{Structure of one set in the sparsity pattern from
    Proposition \ref{prop:sparsity_pattern}
    for a network with 2D convolutional layers with $3\times3$ filters.}
  \label{fig:conv2d}
\end{minipage}
\end{figure}

\textbf{Remark. }When the network is not dense, the the second condition
(Definition \ref{def:sparsity_pattern}) for the sparsity pattern
\eqref{eq:sparsity_pattern} to be valid might not hold. In that case we lose
the guarantee that the values of the corresponding LPs converge to the maximum
of the POP \eqref{eq:upper_bound}. Nevertheless, it still provides a valid
positivity certificate that we use to upper bound $L(f_d)$. In Section
\ref{sec:experiments} we show that in practice it provides upper bounds of
good enough quality. If needed, a valid sparsity pattern can be obtained via a
chordal completion of the \emph{correlative sparsity graph} of the POP \citep{Waki2006}.

We now quantify how good this sparsity pattern is. Let $s$ be the
size of the largest clique in a sparsity pattern, and let $N_{i,k}$ be the
subset of $N_i$ (defined in \autoref{thm:sparse-krivine}) composed of sequences
summing up to $k$. The number of different polynomials for the $k$-th LP in the
hierarchy given by the sparse Krivine's certificate can be bounded as follows:
\begin{equation}
    \label{eq:upb_sparsity}
    \abs{\bigcup_{i=1}^m N_{i,k}} \leq \sum_{i=1}^m {2\abs{I_i} + k \choose k}
    = \bigo \left(m s^k \right)
\end{equation}
We immediately see that the dependence on the number of cliques $m$ is really mild
(linear) but the size of the cliques as well as the degree of the hierarchy can
really impact the size of the optimization problem. Nevertheless, this upper
bound can be quite loose; polynomials $h_{\alpha\beta}$ that depend only on
variables in the intersection of two or more cliques are counted more than
once.

The number of cliques given in the sparsity pattern induced by $G_d$ is equal
to the size of the last layer $m=n_d$ and the size of each clique depends on
the particular implementation of the network. We now study different
architectures that could arise in practice, and determine the amount of
polynomials in their sparse Krivine's certificate.

\textbf{Fully connected networks. }Even in the case of a network with
all nonzero connections, the sparsity pattern induced by $G_d$
decreases the size of the LPs when compared to Krivine's certificate.
In this case the cliques have size $n_1 + \ldots + n_{d-1} + 1$ but they
all have the same common intersection equal to all neurons up to the
second-to-last hidden layer. A straightforward counting argument shows that
the total number of polynomials is $\bigo(n(n_1+\ldots+n_{d-1}+1)^{k-1})$, improving
the upper bound \eqref{eq:upb_sparsity}.

\textbf{Unstructured sparsity. }In the case of networks obtained by pruning
\citep{Hanson88} or generated randomly from a distribution over graphs
\citep{Xie2019}, the sparsity pattern can be arbitrary. In this case
the size of the resulting LPs varies at runtime. Under the layer-wise
assumption that any neuron is connected to at most $r$ neurons in the previous
layer, the size of the cliques in \eqref{eq:sparsity_pattern}
is bounded as $s=\bigo(r^d)$. This estimate has an exponential dependency on
the depth but ignores that many neurons might share connections to
the same inputs in the previous layer, thus being potentially loose. The
bound \eqref{eq:upb_sparsity} implies that the number of different polynomials
is $\bigo(n_d r^{dk})$.

\textbf{2D Convolutional networks. }The sparsity in the weight matrices of
convolutional layers has a certain \emph{local structure}; neurons are
connected to contiguous inputs in the previous layer. Adjacent neurons also
have many input pixels in common (see Figure \ref{fig:conv2d}). Assuming a
constant number of channels per layer, the size of the cliques in
\eqref{eq:sparsity_pattern} is $\bigo(d^3)$.  Intuitively, such number is
proportional to the volume of the pyramid depicted in Figure \ref{fig:conv2d}
where each dimension depends linearly on $d$. Using \eqref{eq:upb_sparsity} we get
that there are $\bigo(n_d d^{3k})$ different polynomials in the sparse
Krivine's certificate. This is a drastic decrease in complexity when compared to
the unstructured sparsity case.

The use of sparsity in polynomial optimization preceeds
\autoref{thm:sparse-krivine} \citep{Weisser2018}. First studied in the context
of sum-of-squares by \citet{Kojima2005} and further refined in
\citet{Waki2006,Lasserre2006} (and references therein), it has found
applications in safety verification \citep{Yang2016,Zhang2018b}, sensor
localization \citet{Wang2006}, optimal power flow \citep{Ghaddar2015} and many
others. Our work fits precisely into this set of important applications.
\begin{algorithm}
    \caption{\bfmyalgname \, for ELU activations and sparsity pattern}
    \label{alg:poplip-krivine}
    \hspace*{\algorithmicindent} \textbf{Input:} matrices $\{W_i\}_{i=1}^d$, sparsity pattern $\{I_i\}_{i=1}^m$,
    hierarchy degree $k$.
\begin{algorithmic}[1]
    \State $p \leftarrow (2s_0 - 1)^TW_1^T \prod_{i=1}^{d-1}\Diag(s_i)W_{i+1}^T$ \Comment{compute norm-gradient polynomial}
    \State $x \leftarrow [s_0, \ldots, s_{d-1}]$
    \State $b \leftarrow [b_\gamma: \gamma \in \N^n_k]$ where
    $p(x)=\sum_{\gamma \in \N^n_k} b_\gamma x^\gamma$
    \Comment{compute coefficients of $p$ in basis}

    \For{$i=1,\ldots,m$}
    \State $N_{i,k} \leftarrow \{ (\alpha, \beta) \in \N^{2n}_k: \supp(\alpha)
        \cap \supp(\beta) \subseteq I_i\}$
    \EndFor

    \State $\widetilde{N}_k \leftarrow \cup_{i=1}^m N_{i,k}$
    \State $h \leftarrow \sum_{(\alpha,\beta) \in \widetilde{N}}
    c_{\alpha\beta} h_{\alpha\beta}$ \Comment{compute positivity certificate}

    \State $c \leftarrow [c_{\alpha\beta}:(\alpha,\beta) \in \widetilde{N}_k]$; $\,$
    $y \leftarrow [\lambda, c]$ \Comment{linear program variables}

    \State $Z \leftarrow [z_\gamma]_{\gamma \in \N^n_k}$ where
    $\lambda - h(x)= \sum_{\gamma \in \N_k^n} (z_\gamma^T y) x^\gamma$
    \Comment{compute coefficients of $\lambda - h$ in basis}
    \Return $\min\{\lambda: b = Zy, \, y = [\lambda, c],\, c \geq 0\}$ \Comment{solve LP}
\end{algorithmic}
\end{algorithm}
%
%

\section{QCQP reformulation and Shor's SDP relaxation}
\label{sec:qcqp}
Another way of upper bounding $L(f_d)$ comes from a further
relaxation of \eqref{eq:upper_bound} to an SDP. We consider the following
equivalent formulation where the variables $s_i$ are normalized to lie in the
interval $[-1,1]$, and we rename $t=s_0$:
\begin{equation}
    \label{eq:upper_bound2}
    L(f_d) \leq    \max \left \{ \dfrac{1}{2^{d-1}}  
        s_0^T W_1^T \prod_{i=1}^{d-1} \Diag(s_i + 1)  W_{i+1}^T:
        -1 \leq s_i \leq 1
    \right \}
\end{equation}
Any polynomial optimization problem like \eqref{eq:upper_bound2} can be cast as
a (possibly non-convex) \textit{quadratically constrained quadratic program}
(QCQP) by introducing new variables and quadratic constraints. This is a
well-known procedure described in \citet[Section 2.1]{Park2017}. When $d=2$
problem \eqref{eq:upper_bound2} is already a QCQP (for the $\ell_\infty$ and
$\ell_2$-norm cases) and no modification is necessary.

\textbf{QCQP reformulation. }We illustrate the case $d=3$ where we have the
variables $s_1,s_2$ corresponding to the first and second hidden layer and a
variable $s_0$ corresponding to the input. The norm-gradient polynomial
in this case is cubic, and it can be rewritten as
a quadratic polynomial by introducing new variables corresponding to the
product of the first and second hidden layer variables.

More precisely the introduction of a variable $s_{1,2}$ with quadratic
constraint $s_{1,2}=\vect(s_1s_2^T)$ allows us to write the objective
\eqref{eq:upper_bound2} as a quadratic polynomial. The problem then
becomes a QCQP with variable $y=[1,s_0,s_1,s_2,s_{1,2}]$ of dimension
$1+n+n_1n_2$.

\textbf{SDP relaxation. }Any quadratic objective and constraints can then be
relaxed to linear constraints on the positive semidefinite variable $yy^T=X
\succcurlyeq 0$ yielding the so-called \emph{Shor's relaxation} of
\eqref{eq:upper_bound2} \citep[Section 3.3]{Park2017}. When $d=2$ the resulting
SDP corresponds precisely to the one studied in \citet{Raghunathan2018a}. This
resolves a common misconception \citep{Raghunathan2018b} that this approach is
only limited to networks with one hidden layer.

Note that in our setting we are only interested in the optimal value rather
than the optimizers, so there is no need to extract a solution for
\eqref{eq:upper_bound2} from that of the SDP relaxation.

\textbf{Drawback. }This approach includes a further relaxation step from
\eqref{eq:upper_bound2}, thus being fundamentally limited in how tightly it can
upper bound the value of $L(f_d)$. Moreover when compared to LP solvers,
off-the-shelf semidefinite programming solvers are, in general, much more
limited in the number of variables they can efficiently handle.

In the case $d=2$ this relaxation provides a constant factor approximation to
the original QCQP \citep{Ye1999}. Further approximation quality results for
such hierarchical optimization approaches to NP-hard problems are out of the
scope of this work.

\textbf{Relation to sum-of-squares. }The QCQP approach might appear
fundamentaly different to the hierarchical optimization approaches to POPs,
like the one described in Section \ref{sec:lp_certificate}. However, it is
known that Shor's SDP relaxation corresponds exactly to the first degree of the
SOS hierarchical SDP solution to the QCQP relaxation \citep{Lasserre2000}.
Thus, the approach in section \ref{sec:lp_certificate} and the one in
this section are, in essence, the same; they only differ in the choice of
polynomial positivity certificate.


\section{Related work}
\label{sec:related}
Estimation of $L(f_d)$ with $\ell_2$-norm is studied by
\citet{Virmaux2018,Combettes2019,Fazlyab2019,Jin2018}. The method
\textbf{SeqLip} proposed in \citet{Virmaux2018} has the drawback of not
providing true upper bounds. It is in fact a heuristic method for solving
\eqref{eq:upper_bound} but which provides no guarantees and thus can not be
used for robustness certification.  In contrast the \textbf{LipSDP} method of
\citet{Fazlyab2019} provides true upper bounds on $L(f_d)$ and in
practice shows superior performance over both \textbf{SeqLip} and
\textbf{CPLip} \citep{Combettes2019}.

Despite the accurate estimation of \textbf{LipSDP}, its formulation is limited
to the $\ell_2$-norm. The only estimate available for other $\ell_p$-norms
comes from the equivalence of norms in euclidean spaces. For instance, we can
obtain an upper bound for the $\ell_\infty$-norm after multiplying the
$\ell_2$ Lipschitz constant upper bound by the square root of the input
dimension. The resulting bound can be rather loose and our experiments in
section \ref{sec:experiments} confirm the issue.  In contrast, our proposed
approach \bfmyalgname \, can acommodate any norm whose unit ball can be described via
polynomial inequalities.

Let us point to one key advantage of \bfmyalgname, compared to \textbf{LipSDP}
\citep{Jin2018,Fazlyab2019}. In the context of robustness certification we are
given a sample $x^\natural$ and a ball of radius $\epsilon$ around it.
Computing an upper bound on the local Lipschitz constant in this subset, rather
than a global one, can provide a larger region of certified robustness. Taking
into account the restricted domain we can refine the bounds in our POP (see
remark in section \ref{sec:introduction}). This potentially yields a tighter
estimate of the local Lipschitz constant. On the other hand, it is not clear
how to include such additional information in \textbf{LipSDP}, which only
computes one global bound on the Lipschitz constant for the unconstrained
network.

\citet{Raghunathan2018a} find an upper bound for $L(f_d)$ with $\ell_\infty$
metric starting from problem \eqref{eq:upper_bound} but only in the context of
one-hidden-layer networks ($d=2$). To compute such bound they use its
corresponding Shor's relaxation and obtain as a byproduct a differentiable
regularizer for training networks. They claim such approach is limited to the
setting $d=2$ but, as we remark in section \ref{sec:qcqp}, it is just a
particular instance of the SDP relaxation method for QCQPs arising from a
polynomial optimization problem. We find that this method fits into
the \bfmyalgname \, framework, using SOS certificates instead of Krivine's. We
expect that the SDP-based bounds described in \ref{sec:qcqp} can also be used
as regularizers promoting robustness.

\citet{Weng2018} provide an upper bound on the local Lipschitz constant
for networks based on a sequence of ad-hoc bounding arguments, which are
particular to the choice of ReLU activation function. In contrast, our
approach applies in general to activations whose derivative is
bounded.

\section{Experiments}
\label{sec:experiments}
We consider the following estimators of $L(f_d)$ with respect to the $\ell_\infty$ norm:
\begin{center}
    \begin{tabular}{ c|p{11cm} } 
        \textbf{Name} & \textbf{Description} \\ 
 \hline
    \textbf{SDP} & Upper bound arising from the solution of the SDP
        relaxation described in Section \ref{sec:qcqp} \\
        \hline
    \textbf{LipOpt-k} & Upper bound arising from the $k$-th degree of
        the LP hierarchy \eqref{eq:lp_prototype} based on the sparse Krivine Positivstellenstatz. \\
        \hline
    \textbf{Lip-SDP} & Upper bound from \citet{Fazlyab2019} multiplied
        $\sqrt{d}$ where $d$ is the input dimension of the network. \\
        \hline
    \textbf{UBP} & Upper bound determined by the product of the layer-wise
        Lipschitz constants with $\ell_\infty$ metric \\
        \hline
    \textbf{LBS} & Lower bound obtained by sampling $50000$ random points around zero, and
        evaluating the dual norm of the gradient
\end{tabular}
\end{center}

\subsection{Experiments on random networks}
\label{sec:exp-random}
We compare the bounds obtained by the algorithms described above on networks
with random weights and either one or two hidden layers. We define the sparsity
level of a network as the maximum number of neurons any neuron in one layer is
connected to in the next layer. For example, the network represented on
Figure~\ref{fig:sparsity_pattern} has sparsity $2$. The non-zero weights of
network's $i$-th layer are sampled uniformly in $[-\frac{1}{\sqrt{n_i}},
\frac{1}{\sqrt{n_i}}]$ where $n_i$ is the number of neurons in layer $i$. 

For different configurations of width and sparsity, we generate $10$ random
networks and average the obtained Lipschitz bounds. For better comparison, we
plot the relative error. Since we do not know the true Lipschitz constant, we
cannot compute the true relative error. Instead, we take as reference the lower
bound given by \textbf{LBS}. Figures~\ref{fig:1_hidden_layer_error} and
~\ref{fig:2_hidden_layer_error} show the relative error, i.e., $(\hat{L} -
L_{LBS}) / L_{LBS}$ where $L_{LBS}$ is the lower bound computed by
$\textbf{LBS}$ and $\hat{L}$ is the estimated upper bound.
Figures~\ref{fig:1_hidden_layer} and ~\ref{fig:2_hidden_layer} in Appendix
\ref{app:random_networks} we show the values of the computed Lipschitz bounds
for $1$ and $2$ hidden layers respectively.

When the chosen degree for \textbf{LiPopt-k} is the smallest as possible, i.e.,
equal to the depth of the network, we observe that the method is already
competitive with the \textbf{SDP} method, especially in the case of $2$ hidden
layers. When we increment the degree by $1$, \textbf{LiPopt-k} becomes
uniformly better than \textbf{SDP} over all tested configurations.  We remark
that the upper bounds given by \textbf{UBP} are too large to be shown in the
plots. Similarly, for the $1$-hidden layer networks, the bounds from
\textbf{LipSDP} are too large to be plotted.

Finally, we measured the computation time of the different methods on each
tested network (Figures~\ref{fig:1_hidden_layer_time} and
~\ref{fig:2_hidden_layer_time}). We observe that the computation time for
\textbf{LiPopt-k} heavily depends on the network sparsity, which reflects the
fact that such structure is exploited in the algorithm. In contrast, the time
required for \textbf{SDP} does not depend on the sparsity, but only on the size
of the network. Therefore as the network size grows (with fixed sparsity level),
\textbf{LipOpt-k} obtains a better upper bound and runs faster.
Also, with our method, we see that it is possible to increase the computation
power in order to compute tighter bounds when required, making it more flexible
than \textbf{SDP} in terms of computation/accuracy tradeoff. \textbf{LiPopt}
uses the Gurobi LP solver, while \textbf{SDP} uses Mosek. All methods run
on a single machine with Core i7 2.8Ghz quad-core processor and 16Gb of RAM.

\begin{figure}[h!]
\centering
\begin{minipage}{.23\textwidth}
  \centering
    \includegraphics[width=1\textwidth]{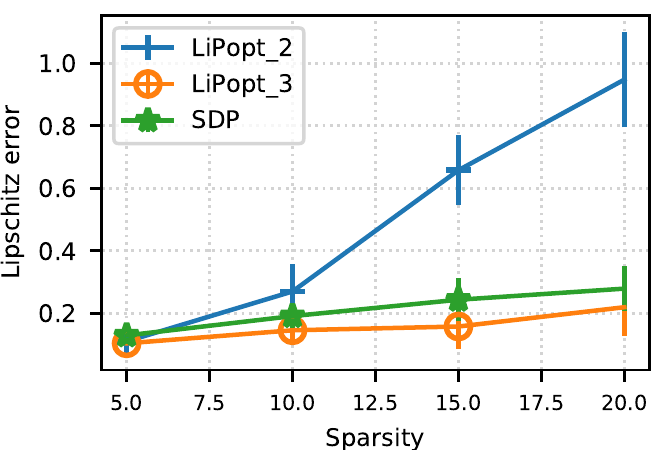}
    \subcaption{$40\times40$}
\end{minipage}%
\begin{minipage}{.23\textwidth}
  \centering
    \includegraphics[width=1\textwidth]{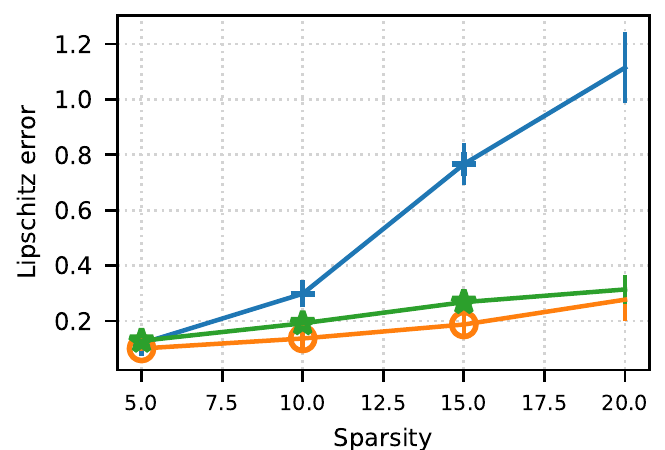}
    \subcaption{$80\times80$}
\end{minipage}
\begin{minipage}{.23\textwidth}
  \centering
    \includegraphics[width=1\textwidth]{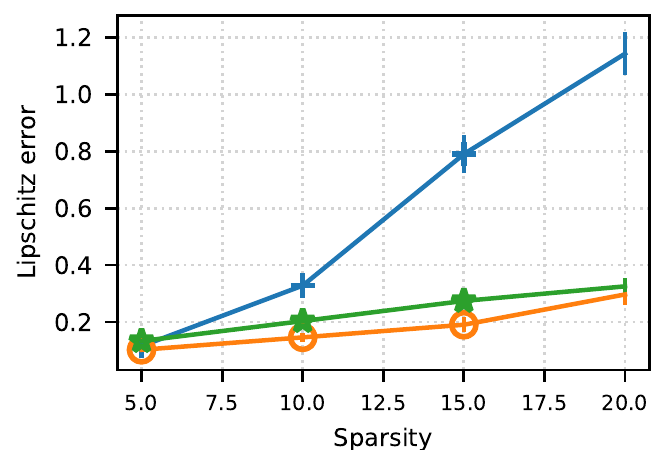}
    \subcaption{$160\times160$}
\end{minipage}
\begin{minipage}{.23\textwidth}
  \centering
    \includegraphics[width=1\textwidth]{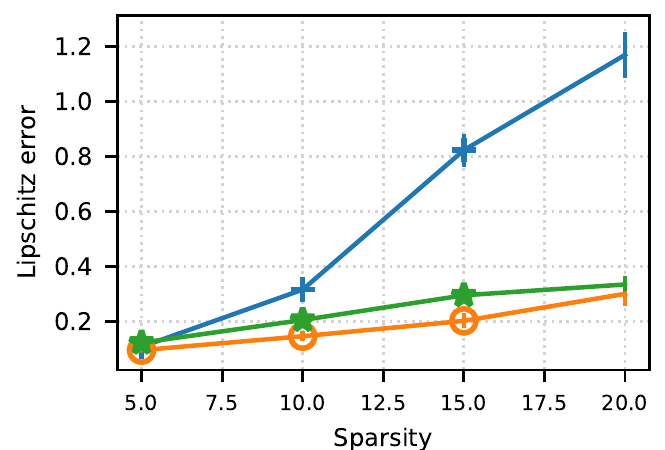}
    \subcaption{$320\times320$}
\end{minipage}
\caption{Lipschitz approximated relative error for $1$-hidden layer networks}
  \label{fig:1_hidden_layer_error}
\end{figure}
\begin{figure}[h!]
\centering
\begin{minipage}{.23\textwidth}
  \centering
    \includegraphics[width=1\textwidth]{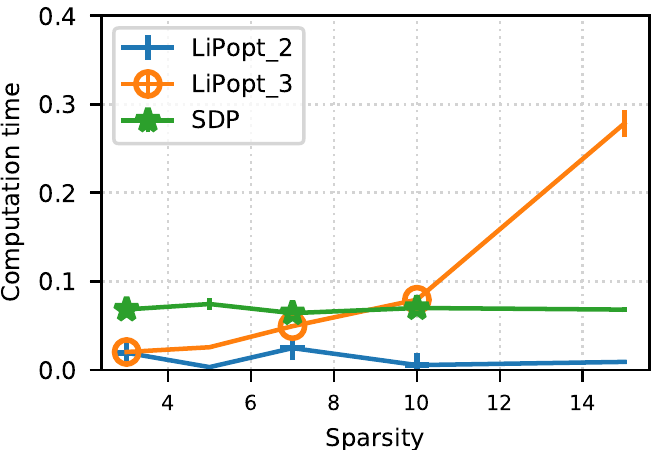}
    \subcaption{$40\times40$}
\end{minipage}%
\begin{minipage}{.23\textwidth}
  \centering
    \includegraphics[width=1\textwidth]{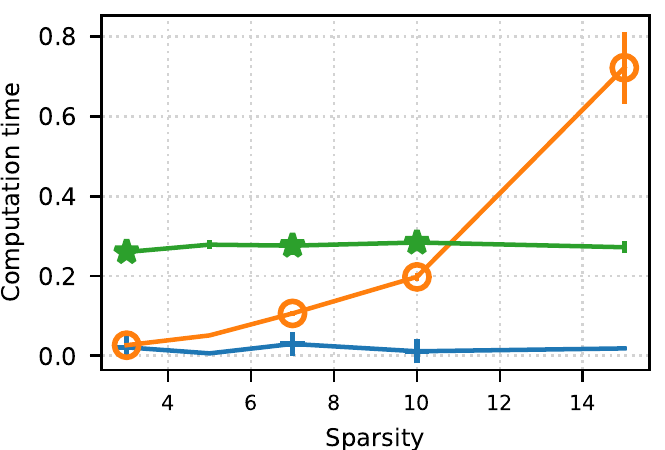}
    \subcaption{$80\times80$}
\end{minipage}
\begin{minipage}{.23\textwidth}
  \centering
    \includegraphics[width=1\textwidth]{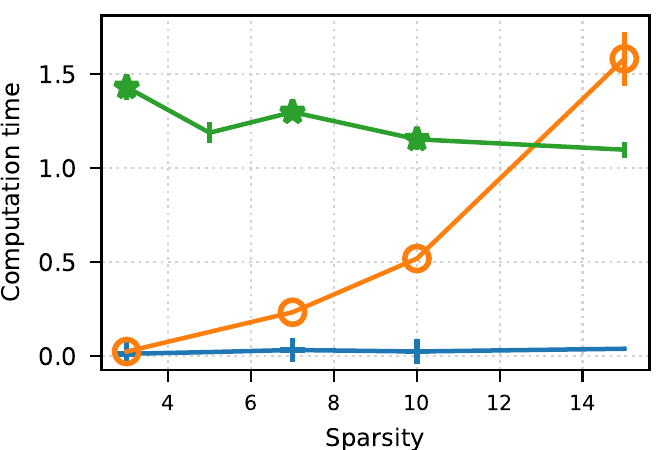}
    \subcaption{$160\times160$}
\end{minipage}
\begin{minipage}{.23\textwidth}
  \centering
    \includegraphics[width=1\textwidth]{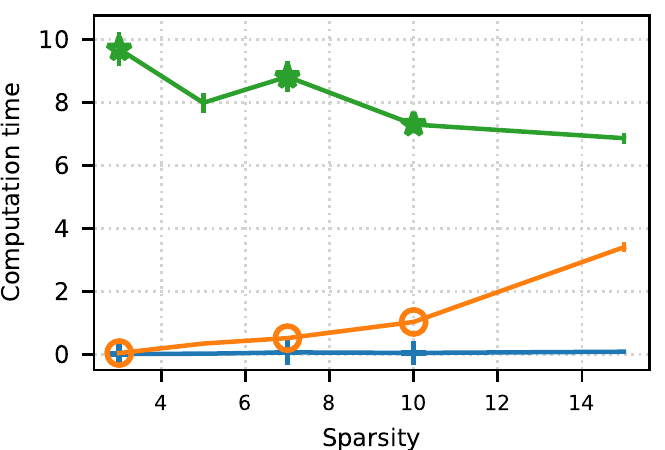}
    \subcaption{$320\times320$}
\end{minipage}
    \caption{Computation times for $1$-hidden layer networks (seconds}
  \label{fig:1_hidden_layer_time}
\end{figure}

\begin{figure}[h!]
\centering
\begin{minipage}{.23\textwidth}
  \centering
    \includegraphics[width=1\textwidth]{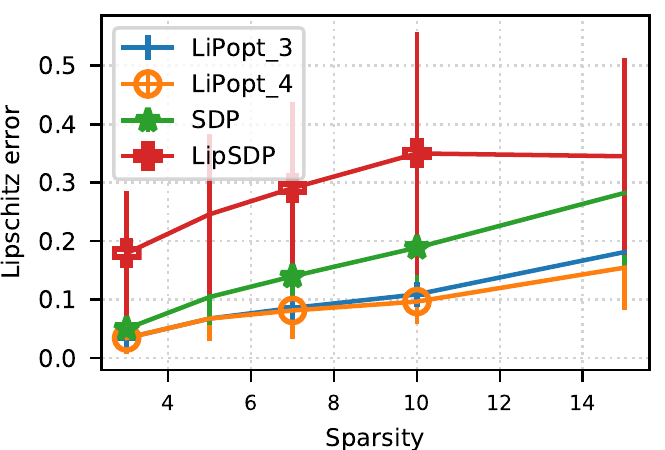}
    \subcaption{$5\times5\times10$}
\end{minipage}%
\begin{minipage}{.23\textwidth}
  \centering
    \includegraphics[width=1\textwidth]{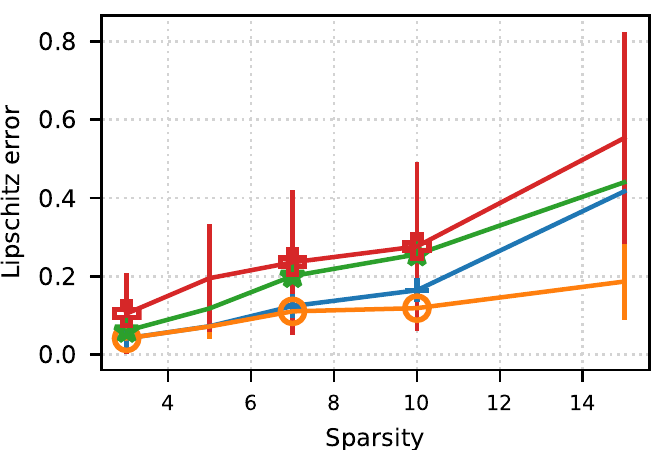}
    \subcaption{$10\times10\times10$}
\end{minipage}
\begin{minipage}{.23\textwidth}
  \centering
    \includegraphics[width=1\textwidth]{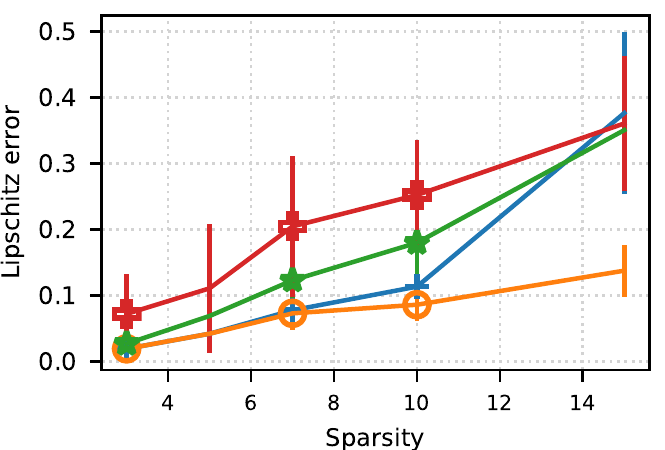}
    \subcaption{$20\times20\times10$}
\end{minipage}
\begin{minipage}{.23\textwidth}
  \centering
    \includegraphics[width=1\textwidth]{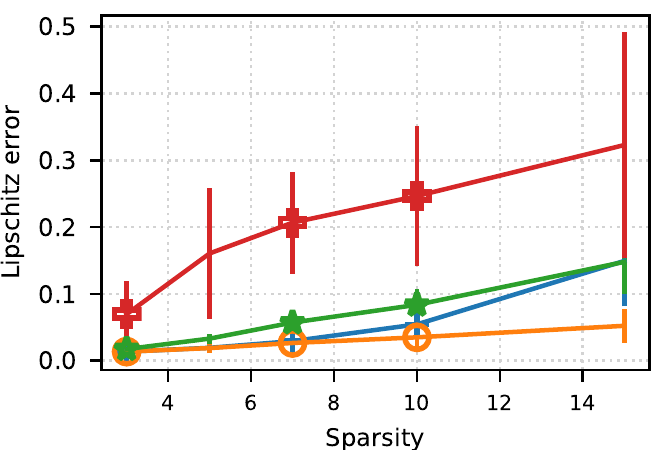}
    \subcaption{$40\times40\times10$}
\end{minipage}
\caption{Lipschitz approximated relative error for $2$-hidden layer networks}
 \label{fig:2_hidden_layer_error}
\end{figure}
\begin{figure}[h!]
\centering
\begin{minipage}{.23\textwidth}
  \centering
    \includegraphics[width=1\textwidth]{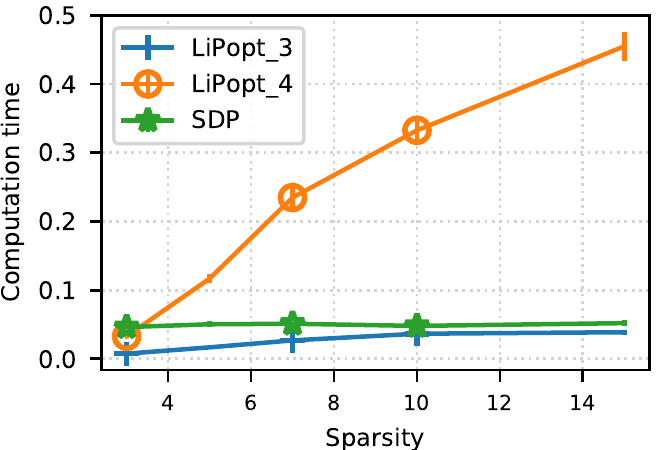}
    \subcaption{$5\times5\times10$}
\end{minipage}%
\begin{minipage}{.23\textwidth}
  \centering
    \includegraphics[width=1\textwidth]{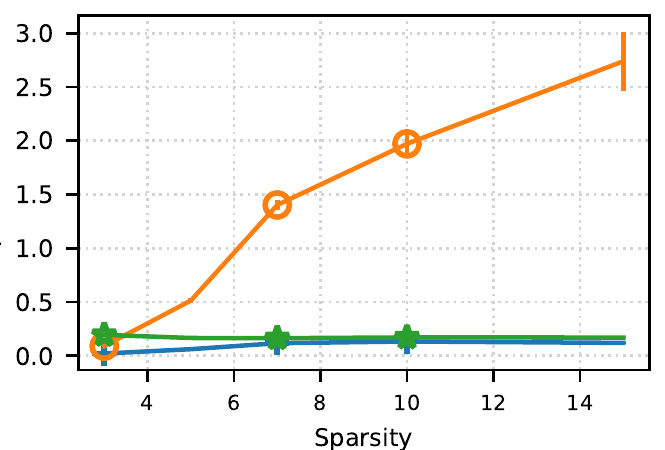}
    \subcaption{$10\times10\times10$}
\end{minipage}
\begin{minipage}{.23\textwidth}
  \centering
    \includegraphics[width=1\textwidth]{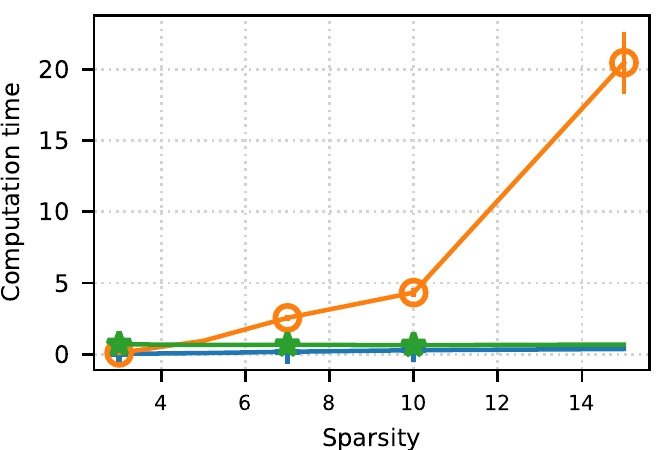}
    \subcaption{$20\times20\times10$}
\end{minipage}
\begin{minipage}{.23\textwidth}
  \centering
    \includegraphics[width=1\textwidth]{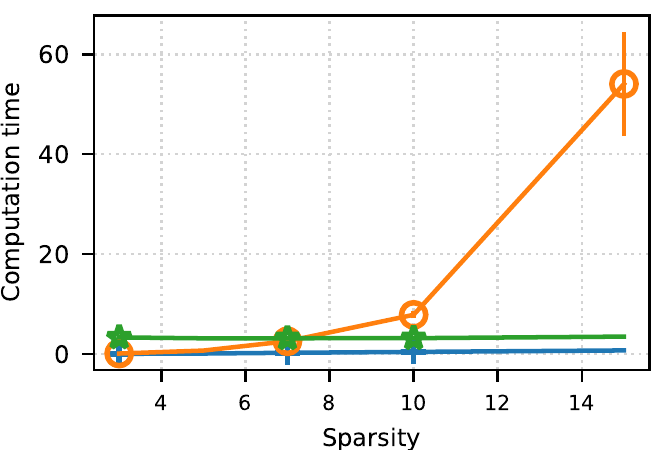}
    \subcaption{$40\times40\times10$}
\end{minipage}
    \caption{Computation times for $2$-hidden layer networks (seconds}
 \label{fig:2_hidden_layer_time}
\end{figure}

\subsection{Experiments on trained networks}

Similarly, we compare these methods on networks trained on MNIST. The architecture we use is a fully connected network with two hidden layers with $300$ and $100$ neurons respectively, and with one-hot output of size $10$. Since the output is multi-dimensional, we restrict the network to a single output, and estimate the Lipschitz constant with respect to label $8$.

Moreover, in order to improve the scalability of our method, we train the
network using the pruning strategy described in
~\cite{han2015deep}\footnote{For training we used the code from this reference.
It is publicly available in
\url{https://github.com/mightydeveloper/Deep-Compression-PyTorch}}. After
training the full network using a standard technique, the weights of smallest
magnitude are set to zero. Then, the network is trained for additional
iterations, only updating the nonzero parameters. Doing so, we were able to
remove $95\%$ of the weights, while preserving the same test accuracy.  We
recorded the Lipschitz bounds for various methods in Table~~\ref{tab:MNIST}. We
observe clear improvement of the Lipschitz bound obtained from
\textbf{LiPopt-k} compared to \textbf{SDP} method, even when using $k=3$. Also
note that the input dimension is too large for the method \textbf{Lip-SDP} to
provide competitive bound, so we do not provide the obtained bound for this
method.

\begin{center}
    \begin{tabular}{ c|c|c|c|c|c } 
        Algorithm & \textbf{LBS} & \textbf{LiPopt-4} & \textbf{LiPopt-3} & \textbf{SDP} & \textbf{UBP} \\ 
 \hline
    Lipschitz bound & $84.2$ & $88.3$ & $94.6$ & $98.8$ & $691.5$ \\ 
\end{tabular}
\label{tab:MNIST}
\end{center}

\if false
\begin{center}
\begin{tabular}{ c|c|c|c|c|c } 
    Algorithm & \textbf{LBS} & \textbf{LiPopt-4} & \textbf{LiPopt-3} & \textbf{LiPopt-4} & \textbf{SDP} & \textbf{UBP} \\ 
 \hline
    Lipschitz bound & $84.2$ & $88.3$ & $94.6$ & $98.8$ & $691.5$ \\ 
\end{tabular}
\label{tab:MNIST}
\end{center}
\fi

\subsection{Estimating local Lipschitz constants with LiPopt}
\label{sec:exp-local}
In the of section \ref{sec:exp-random}, we study the improvement
on the upper bound obtained by \textbf{LiPopt}, when we incorporate tighter
upper and lower bounds on the variables $s_i$ of the polynomial optimization
problem \eqref{eq:upper_bound}. Such bounds arise from the limited
range that the pre-activation values of the network can take, when the input is
limited to an $\ell_\infty$-norm ball of radius $\epsilon$ centered at an
arbitrary point $x_0$.

The algorithm that computes upper and lower bounds on the pre-activation values
is fast (it has the same complexity as a forward pass) and is described, for
example, in \citet{Wong2018}. The variables $s_i$ correspond to the value of
the derivative of the activation function. For activations like ELU or ReLU,
their derivative is monotonically increasing, so we need only evaluate it at
the upper and lower bounds of the pre-activation values to obtain
corresponding bounds for the variables $s_i$.

We plot the local upper bounds obtained by \textbf{LiPopt-3} for increasing
values of the radius $\epsilon$, the bound for the global constant
(given by \textbf{LiPopt-3}) and the lower bound on the local
Lipschitz constant obtained by sampling in the $\epsilon$-neighborhood (LBS).
We sample 15 random networks and plot the average values obtained. We observe
clear gap between both estimates, which shows that larger certified balls could
be obtained with such method in the robustness certification applications.
\begin{figure}[h!]
\centering
\begin{minipage}{.23\textwidth}
  \centering
    \includegraphics[width=1\textwidth]{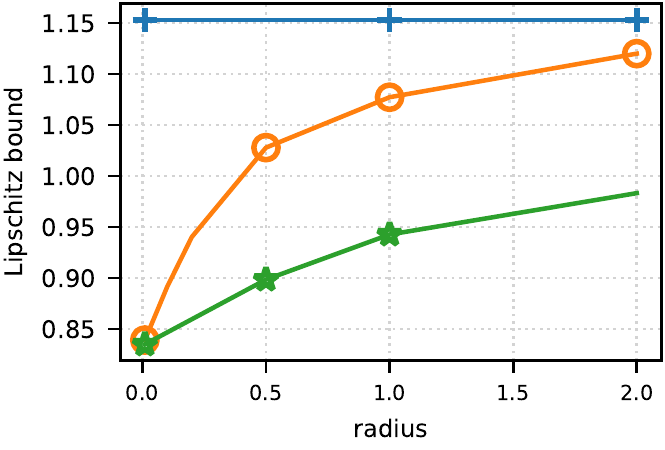}
    \subcaption{$40\times40$}
\end{minipage}%
\begin{minipage}{.23\textwidth}
  \centering
    \includegraphics[width=1\textwidth]{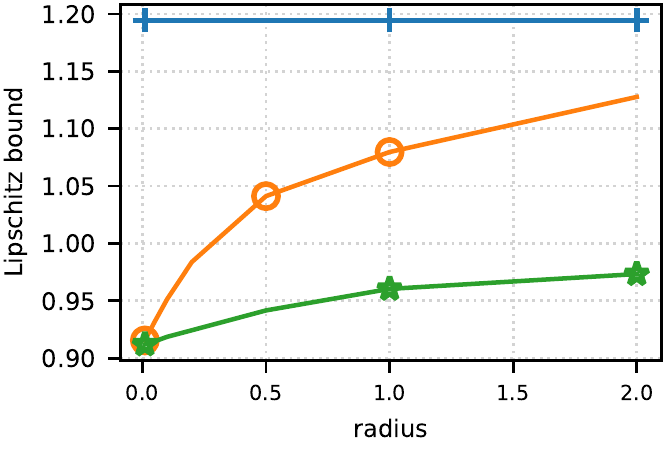}
    \subcaption{$80\times80$}
\end{minipage}
\begin{minipage}{.23\textwidth}
  \centering
    \includegraphics[width=1\textwidth]{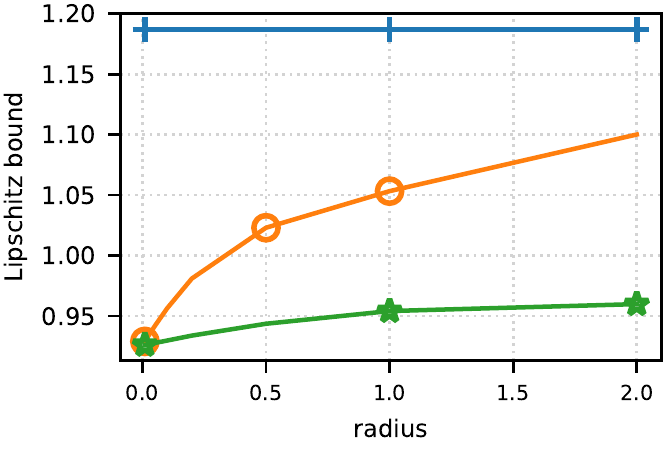}
    \subcaption{$160\times160$}
\end{minipage}
\begin{minipage}{.23\textwidth}
  \centering
    \includegraphics[width=1\textwidth]{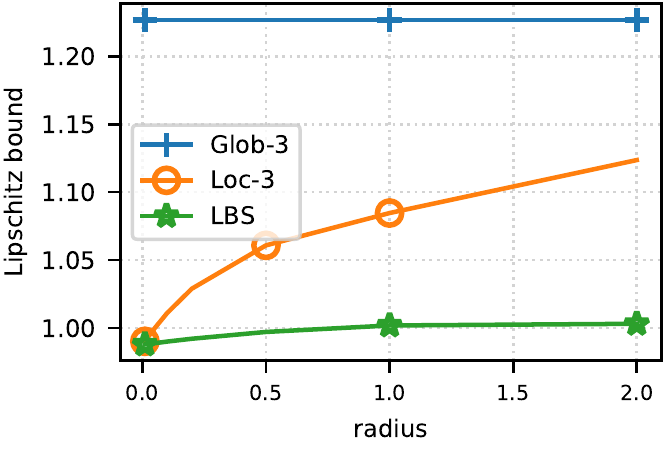}
    \subcaption{$320\times320$}
\end{minipage}
\caption{Global vs local Lipschitz constant bounds for $1$-hidden layer networks}
  \label{fig:1_hidden_layer_local}
\end{figure}

\begin{figure}[h!]
\centering
\begin{minipage}{.23\textwidth}
  \centering
    \includegraphics[width=1\textwidth]{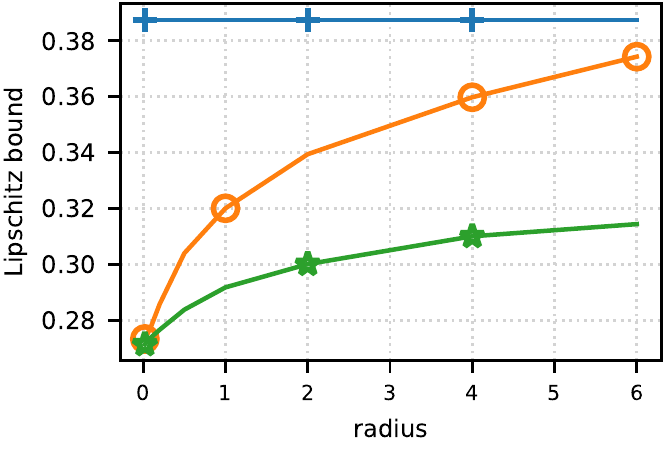}
    \subcaption{$10\times10\times10$}
\end{minipage}%
\begin{minipage}{.23\textwidth}
  \centering
    \includegraphics[width=1\textwidth]{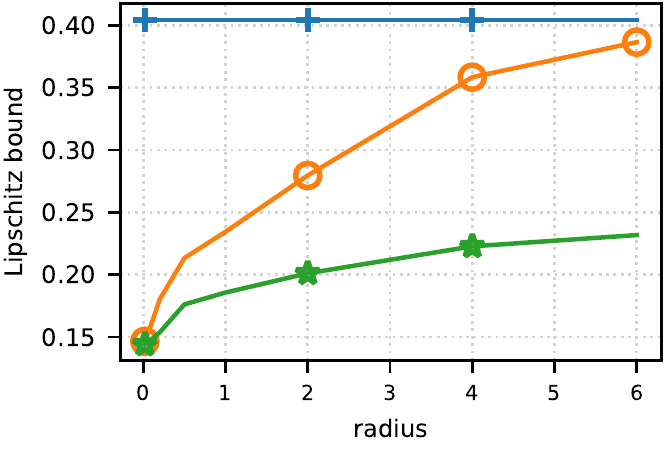}
    \subcaption{$20\times20\times10$}
\end{minipage}
\begin{minipage}{.23\textwidth}
  \centering
    \includegraphics[width=1\textwidth]{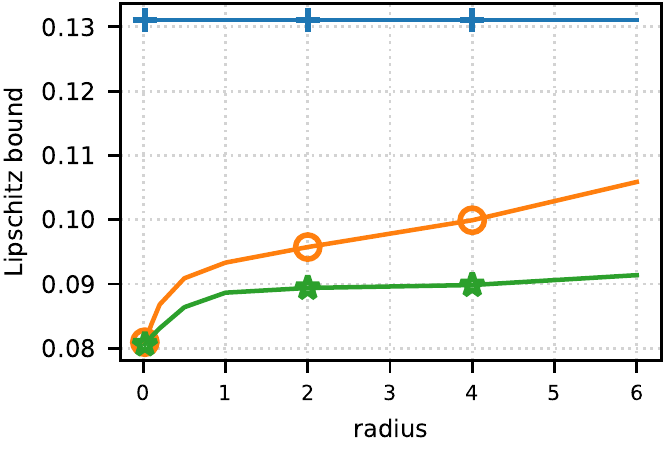}
    \subcaption{$40\times40\times10$}
\end{minipage}
\begin{minipage}{.23\textwidth}
  \centering
    \includegraphics[width=1\textwidth]{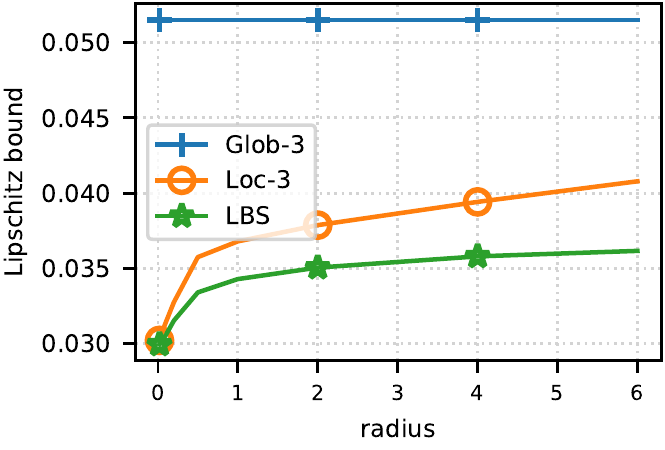}
    \subcaption{$80\times80\times10$}
\end{minipage}
\caption{Global vs local Lipschitz constant bounds for $2$-hidden layer networks}
 \label{fig:2_hidden_layer_local}
\end{figure}

\section{Conclusion and future work}
\label{sec:discussion}
In this work, we have introduced a general approach for computing an upper
bound on the Lipschitz constant of neural networks. This approach is based on
polynomial positivity certificates and generalizes some existing methods
available in the literature. We have empirically demonstrated that it can
tightly upper bound such constant. The resulting optimization problems are
computationally expensive but the sparsity of the network can reduce this
burden.

In order to further scale such methods to larger and deeper networks, we are
interested in several possible directions: ($i$) divide-and-conquer approaches
splitting the computation on sub-networks in the same spirit of
\citet{Fazlyab2019}, $(ii)$ exploiting parallel optimization algorithms
leveraging the structure of the polynomials, ($iii$) custom optimization
algorithms with low-memory costs such as Frank-wolfe-type methods for SDP
\citep{Yurtsever2019} as well as stochastic handling of constraints
\citep{Fercoq2019} and ($iv$), exploting the symmetries in the polynomial that
arise from weight sharing in typical network architectures to further reduce
the size of the problems.

\subsubsection*{Acknowledgments}
This project has received funding from the European Research Council (ERC)
under the European Union's Horizon 2020 research and innovation programme
(grant agreement 725594 - time-data) and from the Swiss National Science
Foundation (SNSF) under  grant number 200021\_178865. FL is supported through a
PhD fellowship of the Swiss Data Science Center, a joint venture between EPFL
and ETH Zurich. VC acknowledges the 2019 Google Faculty Research Award.

\bibliography{references}
\bibliographystyle{iclr2020_conference}

\newpage
\appendix
\section{Proof of \autoref{thm:main}}
\label{app:lips}
\begin{theorem*}
    Let $f$ be a differentiable and Lipschitz continuous function on an open,
    convex subset $\Dom$ of a euclidean space. Let $\|\cdot \|$ be the dual norm. The
    Lipschitz constant of $f$ is given by
    \begin{equation}
        \label{eq:lips_sup_derivative}
        L(f)=\sup_{x \in \Dom} \norm{\nabla{f}(x)}_*
    \end{equation}
\end{theorem*}
\begin{proof}
    First we show that $L(f) \leq \sup_{x \in \Dom} \norm{\nabla{f}(x)}_*$.
    \begin{align*}
        \abs{f(y)-f(x)} &= \abs{\int_0^1 \nabla f((1-t)x+ty)^T(y-x)\, dt} \\
        & \leq \int_0^1 \abs{\nabla f((1-t)x + ty)^T(y-x)} \, dt \\
        & \leq \int_0^1 \normi{\nabla f((1-t)x + ty)}_*\, dt \, \normi{y-x} \\
        & \leq \sup_{x \in \Dom} \normi{\nabla f(x)}_* \normi{y-x}
    \end{align*}
    were we have used the convexity of $\Dom$.

    Now we show the reverse inequality $L(f) \geq \sup_{x \in \Dom}
    \normi{\nabla{f}(x)}_*$.  To this end, we show that for any positive
    $\epsilon$, we have that $L(f) \geq \sup_{x\in\Dom} \normi{\nabla f(x)}_* -
    \epsilon$.

    Let $z \in \Dom$ be such that $\normi{\nabla f(z)}_* \geq \sup_{x \in
    \Dom}\normi{\nabla f(x)}_* - \epsilon$.  Because $\Dom$ is open, there
    exists a sequence $\{h_k\}_{k=1}^\infty$ with the following properties:
    \begin{enumerate}
        \item $\langle h_k, \nabla f(z) \rangle = \normi{h_k} \normi{\nabla f(z)}_*$
        \item $z + h_k \in \Dom$
        \item $\lim_{k \to \infty} h_k = 0$. 
    \end{enumerate}
    By definition of the gradient, there exists a function $\delta$ such that
    $\lim_{h \to 0} \delta(h) = 0$ and the following holds:
    \begin{align*}
         f(z+h) &= f(z) + \langle h, \nabla f(z) \rangle + \delta(h) \normi{h} \\
    \end{align*}
    For our previously defined iterates $h_k$ we then have
    \begin{align*}
        \Rightarrow \abs{f(z+h_k) - f(z)} &= \abs{\normi{h_k} \normi{\nabla f(z)}_* + \delta(h_k)\normi{h_k}}
    \end{align*}
    Dividing both sides by $\normi{h_k}$ and using the definition of $L(f)$ we finally get
    \begin{align*}
        \Rightarrow  L(f) & \geq \abs{\frac{f(z+h_k) - f(z)}{\normi{h_k}}} = \abs{\normi{\nabla f(z)}_* + \delta(h_k)} \\
        \Rightarrow  L(f) & \geq \lim_{k \to \infty} \abs{\normi{f(z)}_* + \delta(h_k)} = \normi{\nabla f(z)}_* \\
        \Rightarrow L(f) & \geq \sup_{x \in \Dom} \normi{\nabla f(x)}_* - \epsilon
    \end{align*}
\end{proof}

\newpage
\section{Proof of Proposition \autoref{prop:sparsity_pattern}}
\label{app:proof_sparsity}
\begin{proposition*}
    Let $f_d$ be a dense network (all weights are nonzero). The following sets,
    indexed by $i=1,\ldots,n_d$, form a valid sparsity pattern for the norm-gradient
    polynomial of the network $f_d$:
    \begin{equation}
        I_{i}:= \left \{s_{(d-1, i)}\} \cup \{s_{(j,k)}: \text{ there exists a
        directed path from }s_{(j,k)} \text{ to }  s_{(d-1,i)} \text{ in } G_d \right \}
    \end{equation}
\end{proposition*}
\begin{proof}
First we show that $\cup_{i=1}^m I_i = I$. This comes from the fact that any
neuron in the network is connected to at least one neuron in the last layer.
Otherwise such neuron could be removed from the network altogether.

Now we show the second property of a valid sparsity pattern. Note that the
norm-gradient polynomial is composed of monomials corresponding to the
product of variables in a path from input to a final neuron. This imples
that if we let $p_i$ be the sum of all the terms that involve the neuron
$s_{(d-1,i)}$ we have that $p = \sum_i p_i$, and $p_i$ only depends on the
variables in $I_i$.

We now show the last property of the valid sparsity pattern. This is the only
part where we use that the network is dense. For any network architecture
the first two conditions hold. We will use the fact that the maximal
cliques of a chordal graph form a valid sparsity pattern (see for example
\citet{Lasserre2006}).
    
Because the network is dense, we see that the clique $I_i$ is composed of the
neuron in the last layer $s_{(d-1, i)}$ and all neurons in the previous
layers. Now consider the extension of the computational graph $\hat{G}_d =(V,
    \hat{E})$ where
    \[
        \hat{E} = E \cup \{(s_{j,k}, s_{l,m}): j,l \leq {d-2})\}
    \]
which consists of adding all the edges between the neurons that are not in the
    last layer. We show that this graph is chordal.  Let $(a_1, \ldots, a_r,
    a_1)$ be a cycle of length at least 4 ($r\geq 4$).  notice that because
    neurons in the last layer are not connected between them in $\hat{G}$, no two
    consecutive neurons in this cycle belong to the last layer. This implies that
    in the subsequence $(a_1, a_2, a_3, a_4, a_5)$ at most three belong to the last
    layer. A simple analysis of all cases implies that it contains at least two nonconsecutive
    neurons not in the last layer. Neurons not in the last layer are always
    connected in $\hat{G}$. This constitutes a chord. This shows that
    $\hat{G}_d$ is a chordal graph. Its maximal cliques correspond exactly to
    the sets in proposition.

\end{proof}

\newpage
\section{Experiments on random networks}
\label{app:random_networks}
\begin{figure}[h!]
\centering
\begin{minipage}{.24\textwidth}
  \centering
    \includegraphics[width=1\textwidth]{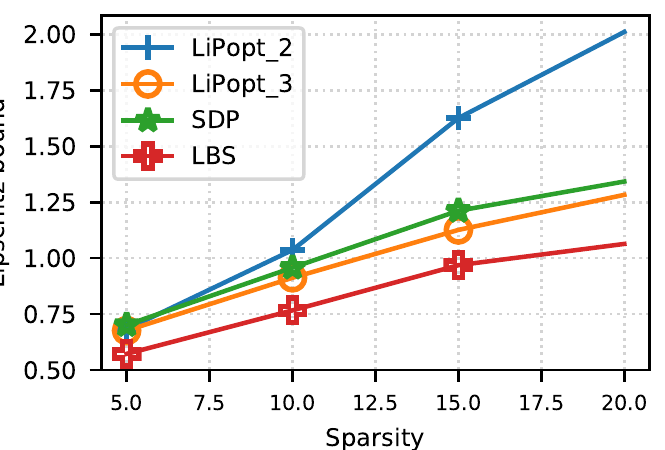}
    \subcaption{$40\times40$}
\end{minipage}%
\begin{minipage}{.24\textwidth}
  \centering
    \includegraphics[width=1\textwidth]{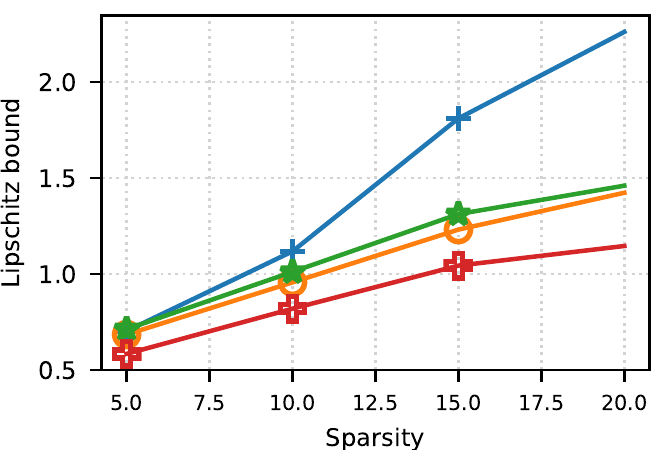}
    \subcaption{$80\times80$}
\end{minipage}
\begin{minipage}{.24\textwidth}
  \centering
    \includegraphics[width=1\textwidth]{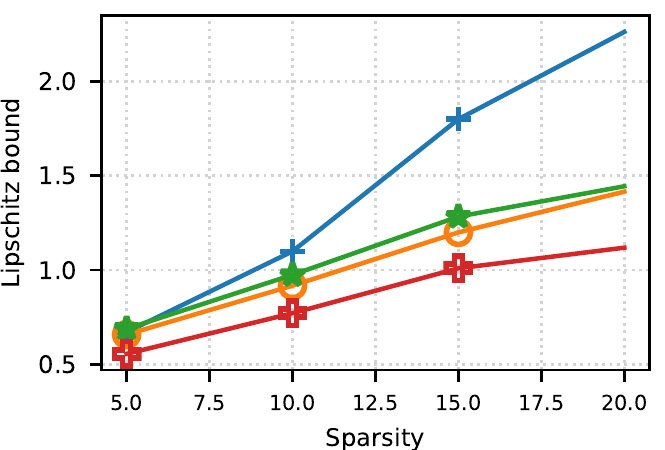}
    \subcaption{$160\times160$}
\end{minipage}
\begin{minipage}{.24\textwidth}
  \centering
    \includegraphics[width=1\textwidth]{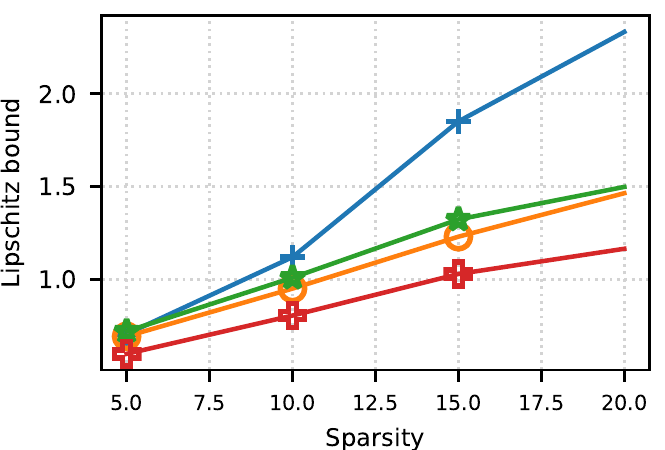}
    \subcaption{$320\times320$}
\end{minipage}
\caption{Lipschitz bound comparison for $1$-hidden layer networks}
  \label{fig:1_hidden_layer}
\end{figure}

\begin{figure}[h!]
\centering
\begin{minipage}{.24\textwidth}
  \centering
    \includegraphics[width=1\textwidth]{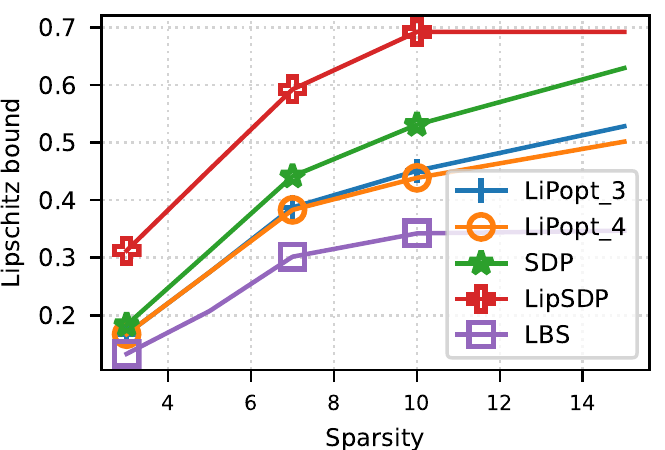}
    \subcaption{$5\times5\times10$}
\end{minipage}%
\begin{minipage}{.24\textwidth}
  \centering
    \includegraphics[width=1\textwidth]{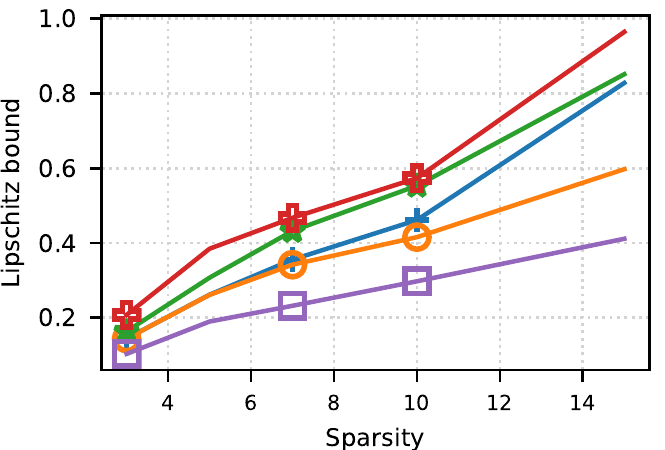}
    \subcaption{$10\times10\times10$}
\end{minipage}
\begin{minipage}{.24\textwidth}
  \centering
    \includegraphics[width=1\textwidth]{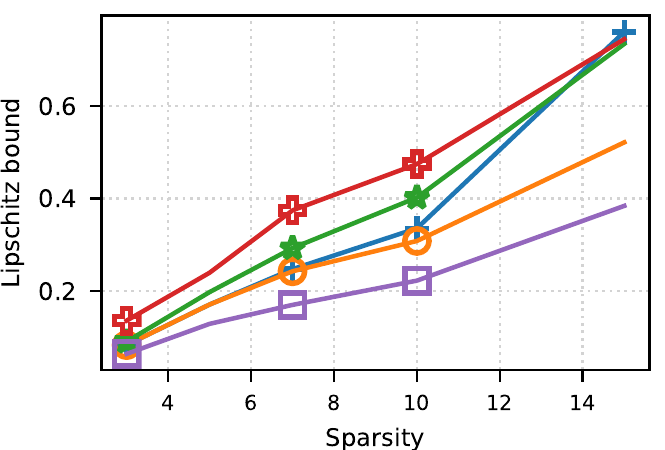}
    \subcaption{$20\times20\times10$}
\end{minipage}
\begin{minipage}{.24\textwidth}
  \centering
    \includegraphics[width=1\textwidth]{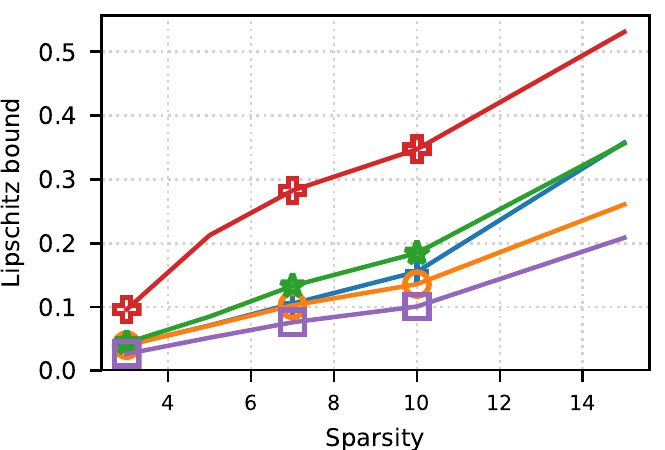}
    \subcaption{$40\times40\times10$}
\end{minipage}
\caption{Lipschitz bound comparison for $2$-hidden layer networks}
 \label{fig:2_hidden_layer}
\end{figure}

%
%

\end{document}